\documentclass[twocolumn,letterpaper]{article}
\usepackage[utf8]{inputenc}

\pdfoutput=1

\usepackage{arxiv}

\usepackage{textcomp}
\usepackage{amsmath,mathtools}
\usepackage{amsthm,thmtools,thm-restate}
\usepackage{txfonts}
\usepackage{amssymb} 
\usepackage[scr=rsfso]{mathalfa}
\usepackage{dsfont}
\usepackage{bm}

\usepackage[boxed]{algorithm}
\usepackage{algpseudocode}

\usepackage[inline,shortlabels]{enumitem}
\usepackage{pgfplots}
\usepackage{pgfplotstable}
\usepackage{tikz-cd}
\usepackage{booktabs}

\usepackage[round,colon,authoryear]{natbib}
\usepackage{varioref}
\usepackage[colorlinks=true,citecolor=blue]{hyperref}
\usepackage[capitalise]{cleveref}

\usepackage[disable]{todonotes}

\newcommand{\tinytodo}[2][]{\todo[size=\tiny]{#2}}
\newcommand{\todoc}[2][]{\tinytodo[color=blue!20, #1]{Cs: #2}} 
\newcommand{\todoy}[2][]{\tinytodo[color=red!20, #1]{Y:\@#2}} 

\newcommand{\Real}{\mathbb{R}}

\newcommand{\Ex}{\mathbb{E}}
\renewcommand{\Pr}{\mathbb{P}}

\DeclareMathOperator*{\argmax}{arg\,max}
\newcommand\given[1][]{\nonscript\:#1\vert\nonscript\:\mathopen{}\allowbreak}
\DeclarePairedDelimiter\paren\lparen\rparen
\DeclarePairedDelimiter\bracket\lbrack\rbrack
\DeclarePairedDelimiter\abs||
\DeclarePairedDelimiterX\set[1]\{\}{%
  \renewcommand\given{\nonscript\:\delimsize\vert\nonscript\:\mathopen{}\allowbreak}
  #1
}

\algnewcommand{\CommentLine}[1]{\Statex $\triangleright$ #1}

\newcommand{\bset}[1]{\left\{#1\right\}}
\newcommand{\eqn}[1]{\setlength{\abovedisplayskip}{0.15cm}\setlength{\belowdisplayskip}{0.15cm}\begin{align}#1\end{align}}
\newcommand{\eq}[1]{\setlength{\abovedisplayskip}{0.15cm}\setlength{\belowdisplayskip}{0.15cm}\begin{align*}#1\end{align*}}

\newcommand{\ind}[1]{\mathbb{1}\!\!\bset{#1}}

\renewcommand{\P}[1]{\mathbb{P}\left\{#1\right\}}

\newcommand{\iset}[1]{\left[#1 \right]}

\newcommand{\E}[1]{\mathbb{E}\left[ #1 \right]}

\theoremstyle{plain}
\newtheorem{theorem}{Theorem}

\newtheorem{corollary}[theorem]{Corollary}
\theoremstyle{definition}
\newtheorem{definition}[theorem]{Definition}

\newtheorem{remark}[theorem]{Remark}

\theoremstyle{remark}

\newcommand{\email}[1]{\small\href{mailto:#1}{\nolinkurl{#1}}}

\title{Conservative Bandits}

\DeclareRobustCommand{\authorlist}{
  \begin{tabular}[t]{cccc}
  Yifan Wu & Roshan Shariff & Tor Lattimore & Csaba Szepesv\'ari \\
  \email{ywu12@ualberta.ca} & \email{rshariff@ualberta.ca} &
  \email{tlattimo@ualberta.ca} & \email{szepesva@ualberta.ca} \\[2ex]
    \multicolumn{4}{c}{\textit{Department of Computing Science, University of Alberta}}
  \end{tabular}}
\author{\authorlist}

\date{}

\begin{document}

\maketitle{}

\begin{abstract}
  We study a novel multi-armed bandit problem that models the
  challenge faced by a company wishing to explore new strategies to
  maximize revenue whilst simultaneously maintaining their revenue
  above a fixed baseline, uniformly over time.  While previous work
  addressed the problem under the weaker requirement of maintaining
  the revenue constraint only at a given fixed time in the future, the
  algorithms previously proposed are unsuitable due to their design
  under the more stringent constraints.  We consider both the
  stochastic and the adversarial settings, where we propose, natural,
  yet novel strategies and analyze the price for maintaining the
  constraints.  Amongst other things, we prove both high probability
  and expectation bounds on the regret, while we also consider both
  the problem of maintaining the constraints with high probability or
  expectation.  For the adversarial setting the price of maintaining
  the constraint appears to be higher, at least for the algorithm
  considered.  A lower bound is given showing that the algorithm for
  the stochastic setting is almost optimal.  Empirical results
  obtained in synthetic environments complement our theoretical
  findings.
\end{abstract}

\section{Introduction}
The manager of Zonlex, a fictional company, has just learned about bandit algorithms and 
is very excited about the opportunity to use this advanced technology to maximize Zonlex's revenue by optimizing the content 
on the landing page of the company's website.
Every click on the content of their website pays a small reward; thanks to the high traffic
that Zonlex's website enjoys, this translates into a decent revenue stream. 
Currently, Zonlex chooses the website's contents using a strategy
designed over the years by its best engineers, but the manager suspects that some alternative strategies could potentially extract
significantly more revenue. The manager is willing to explore bandit algorithms to identify
the winning strategy.
The manager's problem is that Zonlex cannot afford to lose more than 10\% of its current revenue 
during its day-to-day operations and \emph{at any given point in time}, as
Zonlex needs a lot of cash to support its operations.
The manager is aware that standard bandit algorithms experiment
``wildly'', at least initially, and as such
may initially lose too much revenue and jeopardize the company's stable operations.
As a result, the manager is afraid of deploying cutting-edge bandit methods,
but notes that this just seems to be a chicken-and-egg problem:
a learning algorithm cannot explore due to the potential high loss, whereas it must explore to be good in the long run.

The problem described in the previous paragraph is ubiquitous.
It is present, for example, when attempting to learn better human-computer interaction
strategies, say
in dialogue systems or educational games. In these cases a designer may feel that experimenting with 
sub-par interaction strategies could cause more harm than good \citep[e.g.,][]{RieLe08:Dialogue,LiuMaBruPo14}.
Similarly, optimizing a production process in a factory via learning (and experimentation) has much potential
\citep[e.g.,][]{GaRie11:JobShopScheduling}, but deviating too much from established ``best
practices'' will often be considered too dangerous. 
For examples from other domains see the survey paper of \citet{GaFe15:SafeRL}.

Staying with Zonlex, the manager also knows that the standard practice in today's internet companies
is to employ A/B testing on an appropriately small percentage of the traffic for some period of time (e.g., 10\%
in the case of Zonlex).
The manager even thinks that perhaps a best-arm identification strategy from the bandit literature,
such as the recent  lil'UCB method of \citet{JaMaNoBu14:lilUCB}, could be more suitable.
While this is appealing, identifying the best possible option may need too much time even with a good 
learning algorithm
(e.g., this happens when the difference in payoff between the best and second best strategies is small).
One can of course stop earlier, but then the potential for improvement is wasted: 
when to stop then becomes a delicate question on its own.
As Zonlex only plans for the next five years anyway, they could adopt the
 more principled yet quite simple approach of first using their default favorite strategy until enough payoff is collected,
so that in the time remaining of the five years the return-constraint
is guaranteed to hold regardless of the future payoffs.
While this is a solution, the manager suspects that other approaches may exist.
One such potential approach is to discourage a given bandit algorithm from exploring the alternative options, 
while in some way encouraging its willingness to use the default option.
In fact, this approach has been studied recently by \citet{Lat15} (in a slightly more general setting than ours).
However, the algorithm of \citet{Lat15} cannot be guaranteed to maintain the return constraint \emph{uniformly in time}.
It is thus unsuitable for the conservative manager of Zonlex; a modification of the algorithm could possibly meet
this stronger requirement, but it appears that this will substantially increase the worst-case regret.

In this paper we ask whether better approaches than the above
naive one exist in the context of multi-armed bandits, and whether the existing approaches can achieve the best possible regret given the uniform constraint
on the total return.
In particular, our contributions are as follows:
\begin{enumerate*}[label=\textbf{(\roman*)}]
\item Starting from multi-armed bandits, we first formulate what we
  call the family of ``conservative bandit problems''.  As expected
  in these problems, the goal is to design learning algorithms that
  minimize regret under the additional constraint that at any given
  point in time, the total reward (return) must stay above a fixed
  percentage of the return of a fixed default arm, i.e., the return
  constraint must hold \emph{uniformly in time}.  The variants differ
  in terms of how stringent the constraint is (i.e., should the
  constraint hold in expectation, or with high probability?), whether
  the bandit problem is stochastic or adversarial, and whether the
  default arm's payoff is known before learning starts.
\item We analyze the naive build-budget-then-learn strategy described
  above (which we call BudgetFirst) and design a significantly better
  alternative for stochastic bandits that switches between using the
  default arm and learning \citep[using a version of UCB, a simple yet
  effective bandit learning algorithm:][]{AGR95,KR95,ACF02} in a
  ``smoother'' fashion.
\item We prove that the new
  algorithm, which we call Conservative UCB, meets the uniform return
  constraint (in various senses), while it can achieve significantly
  less regret than BudgetFirst.  In particular, while BudgetFirst is
  shown to pay a \emph{multiplicative penalty} in the regret for
  maintaining the return constraint, \todoc{We should argue that the
    penalty is indeed multiplicative, i.e., the upper bound is tight.}
  Conservative UCB only pays an \emph{additive penalty}.  We provide
  both high probability and expectation bounds, consider both high
  probability and expectation constraints on the return, and also
  consider the case when the payoff of the default arm is initially
  unknown.
\item We also prove a lower bound on the best regret given the
  constraint and as a result show that the additive penalty is
  unavoidable; thus Conservative UCB achieves the optimal regret in
  a worst-case sense.  While Unbalanced MOSS of \citet{Lat15}, when
  specialized to our setting, also achieves the optimal regret \citep[as
  follows from the analysis of][]{Lat15}, as mentioned earlier
  it does not maintain the constraint uniformly in time
  (it will explore too much at the beginning of time); it also
  relies heavily on the knowledge of the mean payoff of the default
  strategy.
\item We also consider the \emph{adversarial setting} where we design
  an algorithm similar to Conservative UCB\@: the algorithm uses an
  underlying ``base'' adversarial bandit strategy when it finds that
  the return so far is sufficiently higher than the minimum required
  return.  We prove that the resulting method indeed maintains the
  return constraint uniformly in time and we also prove a
  high-probability bound on its regret.  We find, however, that the
  additive penalty in this case is higher than in the stochastic case.
  Here, the Exp3-$\gamma$ algorithm of \citet{Lat15} is an
  alternative, but again, this algorithm is not able to maintain the
  return constraint uniformly in time.
\item The theoretical analysis is complemented by synthetic
  experiments on simple bandit problems whose purpose is to validate
  that the newly designed algorithm is reasonable and to show that the
  algorithms' behave as dictated by the theory developed. We also
  compare our method to Unbalanced MOSS to provide a perspective to
  see how much is lost due to maintaining the return constraint
  uniformly over time.  We also identify future work.  In particular,
  we expect our paper to inspire further works in related, more
  complex online learning problems, such as contextual bandits, or
  even reinforcement learning.
\end{enumerate*}



\subsection{Previous Work}
\label{sec:related}
Our constraint is equivalent to a constraint on the regret to a default strategy, or in the language of prediction-with-expert-advice,
or bandit literature, regret to a default action.
In the full information, mostly studied in the adversarial setting, much work has been devoted to understanding the price
of such constraints \citep[e.g.,]{HP05,EKMW08,Koo13,SNL14}. In particular, \citet{Koo13} studies the Pareto frontier
of regret vectors (which contains the non-dominated worst-case regret vectors of all algorithms).
The main lesson of these works is that in the full information setting 
even a constant regret to a fixed default action can be maintained with essentially no increase in the regret to the best action.
The situation quickly deteriorates in the bandit setting as shown by \citet{Lat15}. 
This is perhaps unsurprising given that, as opposed to the full information setting, 
in the bandit setting one needs to actively explore to get improved estimates
of the actions' payoffs.
As mentioned earlier, \citeauthor{Lat15} describes two learning algorithms relevant to our setting:
In the stochastic setting we consider, Unbalanced MOSS (and its relative, Unbalanced UCB) 
are able to achieve a constant regret penalty while maintaining the return constraint
while Exp3-$\gamma$ achieves a much better regret as compared to our strategy for the adversarial setting. \todoc{How much better?}
However, \emph{neither of these algorithms maintain the return constraint uniformly in time}.
Neither will the constraint hold with high probability. While Unbalanced UCB achieves problem-dependent bounds,
it has the same issues as Unbalanced MOSS with maintaining the return constraint.
Also, all these strategies rely heavily on knowing the payoff of the default action.

More broadly, the issue of staying safe while exploring has long been recognized in reinforcement learning (RL).
\citet{GaFe15:SafeRL} provides a comprehensive survey of the relevant literature.
Lack of space prevents us from including much of this review. However, the short summary is that while the issue
has been considered to be important, no previous approach addresses the problem from a theoretical angle.
Also, while it has been recognized that adding constraints on the return is one way to ensure safety, 
as far as we know, maintaining the constraints during learning (as opposed to imposing them as a way of restricting
the set of feasible policies) has not been considered in this literature.
Our work, while it considers a much simpler setting, suggest a novel approach to address the safe exploration
problem in RL. 

Another line of work considers safe exploration in the related context of optimization \citep{Sui2015}.
However, the techniques and the problem setting (e.g., objective)
in this work is substantially different from ours.


\section{Conservative Multi-Armed Bandits}
\label{sec:bandits}

The multi-armed bandit problem is a sequential decision-making task in
which a learning agent repeatedly chooses an action (called an
\emph{arm}) and receives a reward corresponding to that action.  We
assume there are $K+1$ arms and denote the arm chosen by the agent in
round $t\in\set{1,2,\dotsc}$ by $I_t\in\set{0,\dotsc,K}$.  There is a
reward $X_{t,i}$ associated with each arm $i$ at each round $t$ and the
agent receives the reward corresponding to its chosen arm,
$X_{t,I_t}$.  The agent does not observe the other rewards $X_{t,j}$
($j\neq I_t$).

The learning performance of an agent over a time horizon $n$ is
usually measured by its \emph{regret}, which is the difference between
its reward and what it could have achieved by consistently choosing
the single best arm in hindsight:
\begin{equation}
  \label{eq:regret}
  R_n = \max_{i \in\set{0,\dotsc,K}} \sum_{t=1}^n X_{t,i} - X_{t,I_t}.
\end{equation}
An agent is failing to learn unless its regret grows sub-linearly:
$R_n\in o(n)$; good agents achieve $R_n\in O(\sqrt{n})$ or even
$R_n\in O(\log n)$.

We also use the notation $T_i(n) = \sum_{t=1}^n\ind{I_t=i}$ for the
number of times the agent chooses arm $i$ in the first $n$ time steps.

\subsection{Conservative Exploration}

Let arm $0$ correspond to the conservative default action with the
other arms $1,\dotsc,K$ being the alternatives to be explored.  We
want to be able to choose some $\alpha>0$ and constrain the learner to
earn at least a $1-\alpha$ fraction of the reward from simply playing
arm 0:
\begin{equation}
  \label{eq:constraint}
  \sum_{s=1}^t X_{s,I_s} \geq (1-\alpha)\sum_{s=1}^t X_{s,0}
  \quad\text{for all } t\in\set{1,\dotsc,n}.
\end{equation}
For the introductory example above $\alpha=0.1$, which corresponds to
losing at most 10\% of the revenue compared to the default website.
It should be clear that small values of $\alpha$ force the learner to
be highly conservative, whereas larger $\alpha$ correspond to a weaker
constraint.

We introduce a quantity $Z_n$, called the \emph{budget}, which
quantifies how close the constraint~\eqref{eq:constraint} is to being
violated:
\begin{equation}
  \label{eq:budget}
  Z_t = \sum_{s=1}^t X_{s,I_s} - (1-\alpha)X_{s,0};
\end{equation}
the constraint is satisfied if and only if $Z_t\geq 0$ for all
$t\in\set{1,\dotsc,n}$.  Note that the constraints must hold uniformly
in time.

Our objective is to design algorithms that minimize the
regret~\eqref{eq:regret} while simultaneously satisfying the
constraint~\eqref{eq:constraint}.  In the following sections, we will
consider two variants of multi-armed bandits: the stochastic setting
in \cref{sec:stochastic} and the adversarial setting in
\cref{sec:adversarial}.  In each case we will design algorithms that
satisfy different versions of the constraint and give regret
guarantees.

One may wonder: what if we only care about $Z_n\ge 0$ instead of
$Z_t\ge 0$ for all $t$. Although our algorithms are designed for
satisfying the anytime constraint on $Z_t$ our lower bound, which is
based on $Z_n\ge 0$ only, shows that in the stochastic setting we
cannot improve the regret guarantee even if we only want to satisfy
the overall constraint $Z_n\ge 0$.


\section{The Stochastic Setting}
\label{sec:stochastic}

\begin{figure}
  \begin{center}
    \begin{tikzpicture}[font=\small]
      \pgfmathsetmacro{\xlim}{8}
      \pgfmathsetmacro{\ylim}{6}
      \pgfmathsetmacro{\oneminusalpha}{0.6}
      \pgfmathsetmacro{\ta}{0.55}
      \pgfmathsetmacro{\tb}{0.70}
      \pgfmathsetmacro{\saferatio}{1.5}
      \pgfmathsetmacro{\unsaferatio}{0.5}
      \draw[thick, ->] (0,0) -- (0,\ylim) node[anchor=west]{Cumulative reward};
      \draw[thick, ->] (0,0) -- (\xlim,0) node[anchor=north east]{Rounds};
      
      \draw[semithick, dotted] (\ta*\xlim,\ylim) -- (\ta*\xlim,0) node[anchor=north]{$t-1$};
      \draw[semithick, dotted] (\tb*\xlim,\ylim) -- (\tb*\xlim,0) node[anchor=north]{$t$};
      
      \draw[very thick, dashed,color=red!80!black] (0,0) -- node[sloped, anchor=center,
      below, pos=0.3]{\textbf{Constraint:} $(1-\alpha)\mu_0t$} (\xlim,\ylim*\oneminusalpha);
      
      \draw[ultra thick, dotted, color=green!80!black] (0,0) -- node[sloped, anchor=center,
      above, very near end]{\textbf{Default action:} $\mu_0t$} (\xlim,\ylim);
      
      \draw[ultra thick, color=blue] (0,0) -- node[sloped,
      anchor=center, above]{\textbf{Following default action}} (\ta*\xlim,\ta*\ylim);
      
      \draw[very thick, color=red!80!black] (\ta*\xlim,\ta*\ylim) --
      (\tb*\xlim,\tb*\ylim*\oneminusalpha*\unsaferatio) node[anchor=west]{Unsafe action};
      
      \draw[ultra thick, color=blue] (\ta*\xlim,\ta*\ylim) --
      (\tb*\xlim,\tb*\ylim*\oneminusalpha*\saferatio) node[anchor=west]{Safe action};
      
      \draw[ultra thick,|-|,dotted] (\ta*\xlim,\ta*\ylim) --
      node[midway, near end, anchor=east, text width=1.5cm,
      align=right]{$\widetilde{Z}_{t-1}$ Budget}
      (\ta*\xlim,\ta*\ylim*\oneminusalpha);
      
      \draw[ultra thick,|-|,dotted,color=blue] (\tb*\xlim,\tb*\ylim*\oneminusalpha*\saferatio) --
      node[midway, anchor=west]{$\widetilde{Z}_t>0$}
      (\tb*\xlim,\tb*\ylim*\oneminusalpha);
      
      \draw[ultra thick,|-|,dotted,color=red!80!black] (\tb*\xlim,\tb*\ylim*\oneminusalpha*\unsaferatio) --
      node[midway, anchor=west]{$\widetilde{Z}_t<0$}
      (\tb*\xlim,\tb*\ylim*\oneminusalpha);
    \end{tikzpicture}
  \end{center}
  \caption{\label{fig:budget}Choosing the default arm increases the
    budget.  Then it is safe to explore a non-default arm if it cannot
    violate the constraint (i.e.\ make the budget negative).}
\end{figure}
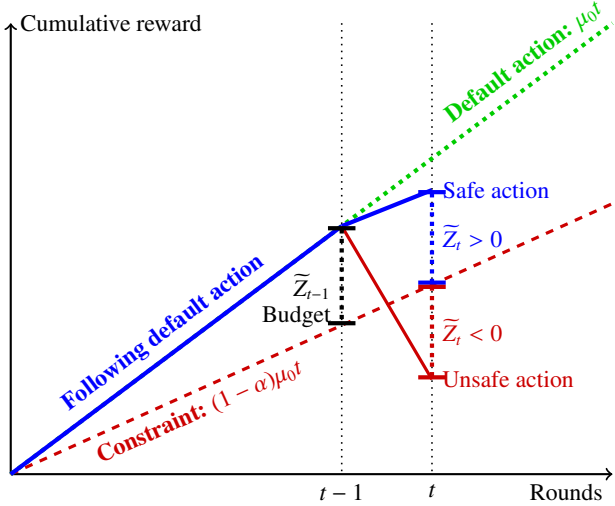

In the stochastic multi-armed bandit setting each arm $i$ and round
$t$ has a stochastic reward $X_{t,i} = \mu_i + \eta_{t,i}$, where
$\mu_i\in[0,1]$ is the expected reward of arm $i$ and the $\eta_{t,i}$
are independent random noise variables that we assume have
1\nobreakdash-subgaussian distributions.  We denote the expected
reward of the optimal arm by $\mu^*=\max_i \mu_i$ and the gap between
it and the expected reward of the $i$th arm by
$\Delta_i = \mu^* - \mu_i$.

The regret $R_n$ is now a random variable.  We can bound it in
expectation, of course, but we are often more interested in
high-probability bounds on the weaker notion of \emph{pseudo-regret:}
\begin{equation}
  \label{eq:pseudoregret}
  \widetilde{R}_n = n\mu^* - \sum_{t=1}^n \mu_{I_t} = \sum_{i=0}^K T_i(n)\Delta_i,
\end{equation}
in which the noise in the arms' rewards is ignored and the randomness
arises from the agent's choice of arm.  The regret $R_n$ and the
pseudo-regret $\widetilde{R}_n$ are equal in expectation.
High-probability bounds for the latter, however, can capture the risk
of exploration without being dominated by the variance in the arms'
rewards.

We use the notation
$\hat\mu_i(n) = \frac{1}{T_i(n)} \sum_{t=1}^n\ind{I_t=i} X_{t,i}$ for
the empirical mean of the rewards from arm $i$ observed by the agent
in the first $n$ rounds.  If $T_i(n)=0$ then we define
$\hat\mu_i(n) = 0$.  The algorithms for the stochastic setting will
estimate the $\mu_i$ by $\hat\mu_i$ and will construct and act based
on high-probability confidence intervals for the estimates.

\subsection{The Budget Constraint}

Just as we substituted regret with pseudo-regret, in the stochastic
setting we will use the following form of the
constraint~\eqref{eq:constraint}:
\begin{equation}
  \label{eq:stoch-constraint}
  \sum_{s=1}^t \mu_{I_s} \geq (1-\alpha)\mu_0 t
  \qquad\text{for all } t\in\set{1,\dotsc,n};
\end{equation}
the budget then becomes
\begin{equation}
  \label{eq:stoch-budget}
  \widetilde{Z}_t = \sum_{s=1}^t \mu_{I_s} - (1-\alpha)t\mu_0 \,.
\end{equation}
The default arm is always safe to play because it increases the budget
by $\mu_0 - (1-\alpha)\mu_0 = \alpha\mu_0$.  The budget will decrease
for arms $i$ with $\mu_i < (1-\alpha)\mu_0$; the constraint
$\widetilde{Z}_n\geq 0$ is then in danger of being violated
(\cref{fig:budget}).

In the following sections we will construct algorithms that satisfy
pseudo-regret bounds and the budget constraint~\eqref{eq:stoch-constraint} with high
probability $1-\delta$ (where $\delta>0$ is a tunable parameter).
In \cref{sec:expectation-bounds} we will see how these algorithms can
be adapted to satisfy the constraint in expectation and with bounds on
their expected regret.

For simplicity, we will initially assume that the algorithms know
$\mu_0$, the expected reward of the default arm.  This is reasonable
in situations where the default action has been used for a long time
and is well-characterized.  Even so, in \cref{sec:unknown-mu0} we will
see that having to learn an unknown $\mu_0$ is not a great hindrance.

\subsection{BudgetFirst --- A Naive Algorithm}

Before presenting the new algorithm it is worth remarking on the most
obvious naive attempt, which we call the BudgetFirst algorithm.  A
straightforward modification of UCB leads to an algorithm that accepts
a confidence parameter $\delta \in (0,1)$ and suffers regret at most
\begin{equation}
  \label{eq:high-prob-ucb}
  \widetilde{R}_n = O\paren*{\sqrt{Kn \log\left(\frac{\log(n)}{\delta}\right)}}
  = R_{\text{worst}}\,.
\end{equation}
Of course this algorithm alone will not satisfy the
constraint~\eqref{eq:stoch-constraint}, but that can be enforced by
naively modifying the algorithm to deterministically choose $I_t = 0$
for the first $t_0$ rounds where \eq{(\forall\, t_0 \leq t \leq n)
  \quad t \mu_0 - R_{\text{worst}} \geq (1 - \alpha) t \mu_0\,.  }
Subsequently the algorithm plays the high probability version of UCB
and the regret guarantee~\eqref{eq:high-prob-ucb} ensures the
constraint~\eqref{eq:stoch-constraint} is satisfied with high
probability. Solving the equation above leads to
$t_0 = \tilde O(R_{\text{worst}} / \alpha \mu_0)$, and since the
regret while choosing the default arm may be $O(1)$ the worst-case
regret guarantee of this approach is
\begin{equation*}
  \widetilde{R}_n =
  \Omega\left(\frac{1}{\mu_0 \alpha} \sqrt{Kn
      \log\left(\frac{\log(n)}{\delta}\right)}\right)\,.
\end{equation*}
This is significantly worse than the more sophisticated algorithm that
is our main contribution and for which the price of satisfying
\eqref{eq:stoch-constraint} is only an additive term rather than a
large multiplicative factor.

\subsection{Conservative UCB}

A better strategy is to play the default arm only until the
budget~\eqref{eq:stoch-budget} is large enough to start exploring
other arms with a low risk of violating the constraint.  It is safe to
keep exploring as long as the budget remains large, whereas if it
decreases too much then it must be replenished by playing the default
arm.  In other words, we intersperse the exploration of a standard
bandit algorithm with occasional budget-building phases when required.
We show that accumulating a budget does not severely curtail
exploration and thus gives small regret.

Conservative UCB (\cref{alg:cucb}) is based on UCB with the novel
twist of maintaining a positive budget.  In each round, UCB calculates
upper confidence bounds for each arm; let $J_t$ be the arm that
maximizes this calculated confidence bound.  Before playing this arm
(as UCB would) our algorithm decides whether doing so risks the budget
becoming negative.  Of course, it does not know the actual budget
$\widetilde{Z}_t$ because the $\mu_i$ ($i\neq 0$) are unknown;
instead, it calculates a lower confidence bound $\xi_t$
based on confidence intervals for the $\mu_i$.  More precisely, it
calculates a lower confidence bound for what the budget would be if it
played arm $J_t$.  If this lower bound is positive then the constraint
will not be violated as long as the confidence bounds hold.  If so,
the the algorithm chooses $I_t = J_t$ just as UCB would; otherwise it
acts conservatively by choosing $I_t = 0$.

\begin{algorithm}
\caption{Conservative UCB}\label{alg:cucb}
\begin{algorithmic}[1]
\State \textbf{Input: $K$, $\mu_0$, $\delta$, $\psi^\delta(\cdot)$}
\For{$t \in 1,2,\dotsc$}
\vspace{1ex}
\CommentLine{Compute confidence intervals\ldots}
\vspace{0.5ex}
\State $\theta_0(t), \lambda_0(t) \gets \mu_0$ \Comment{\ldots for known $\mu_0$,}
\For{$i \in 1,\dotsc,K$} \Comment{\ldots for other arms,}
\State $\Delta_i(t) \gets \sqrt{\psi^\delta(T_i(t-1))/T_i(t-1)}$
\State $\theta_i(t) \gets \hat \mu_i(t-1) + \Delta_i(t)$
\State $\lambda_i(t) \gets \max\bset{0, \hat \mu_i(t-1) - \Delta_i(t)}$
\EndFor
\State $J_t \gets \argmax_i \theta_i(t)$ \Comment{\ldots and find UCB
  arm.}
\vspace{1ex}
\CommentLine{Compute budget and\ldots}
\vspace{1ex}
\State $\xi_t \gets
\sum_{s=1}^{t-1} \lambda_{I_s}(t)
+ \lambda_{J_t}(t)
- (1 - \alpha)t \mu_0$
\If{$\xi_t \geq 0$}
\State $I_t \gets J_t$ \Comment{\ldots choose UCB arm if safe,}
\Else
\State $I_t \gets 0$ \Comment{\ldots default arm otherwise.}
\EndIf
\vspace{1ex}
\EndFor
\end{algorithmic}
\end{algorithm}

\begin{remark}[Choosing $\psi^\delta$]\label{remark:choosing-psi}
  The confidence intervals in \cref{alg:cucb} are constructed using
  the function $\psi^\delta$.  Let $F$ be the event that for all
  rounds $t\in\{ 1, 2, \ldots\}$ and every action $i\in \iset{K}$, the
  confidence intervals are valid:
  \begin{equation*}
    \abs{\hat\mu_i(t) - \mu_i} \le \sqrt{\frac{\psi^\delta(T_i(t))}{T_i(t)}}\,.
  \end{equation*}
  Our goal is to choose $\psi^\delta(\cdot)$ such that
  \begin{equation}
    \label{eq:hprob}
    \P{F}\ge 1-\delta\,.
  \end{equation} 
  A simple choice is
  \begin{align*}
    \psi^\delta(s)=2\log(Ks^3/\delta),
  \end{align*}
  for which~\eqref{eq:hprob} holds by Hoeffding's inequality and union
  bounds.  The following choice achieve better performance in practice:
  \begin{multline}\label{eq:psi}
    \psi^\delta(s) = \log \max\set{3, \log\zeta} + \log(2e^2 \zeta) \\
    + \frac{\zeta(1 + \log(\zeta))}{(\zeta - 1) \log(\zeta)} \log \log
    (1 + s),
  \end{multline}
  where $\zeta = K/\delta$; it can be seen to achieve~\eqref{eq:hprob}
  by more careful analysis motivated by \citet{Gar13}, %
\end{remark}

Some remarks on \cref{alg:cucb}
\begin{itemize}
\item $\mu_0$ is known, so the upper and lower confidence bounds can
  both be set to $\mu_0$ (line 3).  See \cref{sec:unknown-mu0} for a
  modification that learns an unknown $\mu_0$.
\item The $\max$ in the definition of the lower confidence bound
  $\lambda_i(t)$ (line 7) is because we have assumed $\mu_i \geq 0$
  and so the lower confidence bound should never be less than $0$.
\item $\xi_t$ (line 10) is a lower confidence bound on the
  budget~\eqref{eq:stoch-budget} if action $J_t$ is chosen.  More
  precisely, it is a lower confidence bound on
  \begin{align*}
    \widetilde{Z}_t &= \sum_{s=1}^{t-1} \mu_{I_s} + \mu_{J_t} - (1-\alpha)t\mu_0.
  \end{align*}
\item If the default arm is also the UCB arm ($J_t = 0$) and the
  confidence intervals all contain the true values, then
  $\mu^* = \mu_0$ and the algorithm will choose action $0$ for all
  subsequent rounds, incurring no regret.
\end{itemize}

The following theorem guarantees that Conservative UCB satisfies the
constraint while giving a high-probability upper bound on its regret.

\begin{restatable}{theorem}{thmupper}%
  \label{thm:upper}%
  In any stochastic environment where the arms have expected rewards
  $\mu_i\in[0,1]$ with 1-subgaussian noise, \cref{alg:cucb} satisfies
  the following with probability at least $1-\delta$ and for every
  time horizon $n$:
  \begin{align}
    \sum_{s=1}^t\mu_{I_s}
    &\ge (1-\alpha)\mu_0t
      \qquad\text{for all } t\in\set{1,\dotsc,n}, \tag{\ref{eq:stoch-constraint}}\\
    \widetilde{R}_n
    &\le \sum_{i>0:\Delta_i>0} \paren*{\frac{4L}{\Delta_i}+\Delta_i}
      + \frac{2(K+1)\Delta_0}{\alpha\mu_0} \nonumber\\
    &\qquad\qquad + \frac{6L}{\alpha\mu_0}
      \sum_{i=1}^K \frac{\Delta_0}{\max\{\Delta_i,\Delta_0-\Delta_i\}}, \label{eq:pd-upper} \\
    \widetilde{R}_n
    &\in O\paren*{\sqrt{nKL}+\frac{KL}{\alpha\mu_0}},
      \label{eq:pind-upper}
  \end{align}
  when $\psi^\delta$ is chosen in accordance with
  \cref{remark:choosing-psi} and where $L=\psi^\delta(n)$.
\end{restatable}

Standard unconstrained UCB algorithms achieve a regret of order
$O(\sqrt{nKL})$; Theorem~\ref{thm:upper} tells us that the penalty our
algorithm pays to satisfy the constraint is an extra additive regret
of order $O(KL/\alpha\mu_0)$.

\begin{remark}\label{remark:small-alpha}
  We take a moment to understand how the regret of the algorithm
  behaves if $\alpha$ is polynomial in $1/n$.  Clearly if
  $\alpha\in O(1/n)$ then we have a constant exploration budget and
  the problem is trivially hard.  In the slightly less extreme case
  when $\alpha$ is as small as $n^{-a}$ for some $0<a<1$, the extra
  regret penalty is still not negligible: satisfying the constraint
  costs us $O(n^a)$ more regret in the worst case.
  
  We would argue that the problem-dependent regret
  penalty~\eqref{eq:pd-upper} is more informative than the worst case
  of $O(n^a)$; our regret increases by
  \[
    \frac{6L}{\alpha\mu_0} \sum_{i=1}^K
    \frac{\Delta_0}{\max\{\Delta_i,\Delta_0-\Delta_i\}}.
  \]
  Intuitively, even if $\alpha$ is very small, we can still explore as
  long as the default arm is close-to-optimal (i.e.\ $\Delta_0$ is
  small) and most other arms are clearly sub-optimal (i.e.\ the
  $\Delta_i$ are large).  Then the sub-optimal arms are quickly
  discarded and even the budget-building phases accrue little regret:
  the regret penalty remains quite small.  More precisely, if
  $\Delta_0 \approx n^{-b_0}$ and
  $\min_{i>0:\Delta_i>0} \Delta_i \approx n^{-b}$, then the regret
  penalty is
  \[
    O\left( n^{a+\min\{0,b-b_0\}} \right);
  \]
  small $\Delta_0$ and large $\Delta_i$ means $b-b_0<0$, giving a
  smaller penalty than the worst case of $O(n^a)$.
\end{remark}

\begin{remark}
Curious readers may be wondering if $I_t = 0$ is the only conservative choice when the arm proposed by UCB risks violating the constraint.
A natural alternative would be to use the lower confidence bound $\lambda_i(t)$ by choosing
\eqn{\label{eq:alternative}
I_t = \begin{cases}
J_t\,, & \text{if } \xi_t \geq 0\,; \\
\argmax_i \lambda_i(t)\,, & \text{otherwise}\,.
\end{cases}
}
It is easy to see that if $F$ does not occur, then choosing $\argmax_i \lambda_i(t)$ increases the budget at least as much
as choosing action $0$ while incurring less regret and so this algorithm is preferable to \cref{alg:cucb} in practice.
Theoretically speaking, however, it is possible to show that the improvement is by at most a constant factor so our analysis of the
simpler algorithm suffices. The proof of this claim is somewhat tedious so instead we provide two intuitions:
\vspace{-0.4cm}
\begin{enumerate}
\item The upper bound approximately matches the lower bound in the minimax regime, so any improvement must be relatively small in the minimax sense.
\item Imagine we run the unmodified \cref{alg:cucb} and let $t$ be the
  first round when $I_t \neq J_t$ and where there exists an
  $i > 0$ with $\lambda_i(t) \geq \mu_0$.  If $F$ does not hold, then
  the actions chosen by UCB satisfy \eq{T_i(t) \in
    \Omega\left(\min\bset{\frac{L}{\Delta_i^2}, \max_j
        T_j(t)}\right)\,, } which means that arms are being played in
  approximately the same frequency until they are proving suboptimal
  \citep[for a similar proof, see][]{Lat15b}.  From this it follows
  that once $\lambda_{I_t}(t) \geq \mu_0$ for some $i$ it will not be
  long before either $\lambda_j(t+s) \geq \mu_0$ or
  $T_j(t+s) \geq 4L/\Delta_i^2$ and in both cases the algorithm will
  cease playing conservatively. Thus it takes at most a constant
  proportion more time before the naive algorithm is exclusively
  choosing the arm chosen by UCB\@.
\end{enumerate}
\end{remark}

Next we discuss how small modifications to \cref{alg:cucb} allow it to
handle some variants of the problem while guaranteeing the same order of regret.

\subsection{Considering the Expected Regret and Budget}
\label{sec:expectation-bounds}

\todoy{this is for known time horizon, do we need to say something
  about making it anytime?}%
One may care about the performance of the algorithm in expectation
rather than with high probability, i.e.\ we want an upper bound on
$\E{\widetilde{R}_n}$ and the constraint~\eqref{eq:stoch-constraint}
becomes
\begin{equation}
  \label{eq:expt-bounds}
  \Ex\bracket[\Big]{\sum_{s=1}^t\mu_{I_s}}
  \ge (1-\alpha)\mu_0t, \quad
  \text{for all } t\in\set{1,\ldots,n}.
\end{equation}

We argued in \cref{remark:small-alpha} that if $\alpha\in O(1/n)$ then
the problem is trivially hard; let us assume therefore that
$\alpha\ge c/n$ for some $c>1$.  By running \cref{alg:cucb} with
$\delta=1/n$ and $\alpha'=(\alpha-\delta)/(1-\delta)$ we can achieve
\eqref{eq:expt-bounds} and a regret bound with the same order as in
\cref{thm:upper}.

To show
\eqref{eq:expt-bounds} we have
\begin{align*}
  \Ex\bracket[\Big]{\sum_{s=1}^t\mu_{I_s}}
  &\ge \P{F} \Ex\bracket[\Big]{\sum_{s=1}^t\mu_{I_s} \given[\Big] F}\\
  &\ge (1-\delta)(1-\alpha')\mu_0t
  = (1-\alpha)\mu_0t \,.
\end{align*}
In the upper bound of $\E{R_n}$, we have
\begin{align*}
  \E{R_n} \le \E{R_n|F}+\delta n = \E{R_n|F}+1 \,.
\end{align*}
$\E{R_n|F}$ can be upper bounded by \cref{thm:upper} with two
changes:
\begin{enumerate*}[label=(\roman*)]
\item $L$ becomes $O(\log nK)$ after replacing $\delta$ with
  $1/n$, and
\item $\alpha$ becomes $\alpha'$.
\end{enumerate*}
Since $\alpha'/\alpha \ge 1-1/c$ we get essentially the same order of
regret bound as in \cref{thm:upper}.

\subsection{Learning an Unknown $\mu_0$}
\label{sec:unknown-mu0}

Two modifications to \cref{alg:cucb} allow it to handle the case when
$\mu_0$ is unknown. First, just as we do for the non-default arms, we
need to set $\theta_0(t)$ and $\lambda_0(t)$ based on confidence
intervals.  Second, the lower bound on the budget needs to be set as
\begin{multline}
  \label{eq:unknown-mu0-budget}
  \xi'_t = \sum_{i=1}^K T_i(t-1)\lambda_i(t)+\lambda_{J_t}(t) \\ 
  + (T_0(t-1) - (1-\alpha)t)\theta_0(t) \,.
\end{multline}

\begin{restatable}{theorem}{thmunknownmuzero}\label{thm:unknown-mu0}
  \Cref{alg:cucb}, modified as above to work without knowing $\mu_0$
  but otherwise the same conditions as \cref{thm:upper}, satisfies
  with probability $1-\delta$ and for all time horizons $n$ the
  constraint~\eqref{eq:stoch-constraint} and the regret bound
  \begin{multline}
    \label{eq:unknown-mu0-pd-upper}
    \widetilde{R}_n
    \le \sum_{i:\Delta_i>0} \left(\frac{4L}{\Delta_i}+\Delta_i\right)
    + \frac{2(K+1)\Delta_0}{\alpha\mu_0} \\
    + \frac{7L}{\alpha\mu_0} \sum_{i=1}^K
    \frac{\Delta_0}{\max\{\Delta_i,\Delta_0-\Delta_i\}} \,.
  \end{multline}
\end{restatable}

Theorem~\ref{thm:unknown-mu0} shows that we get the same order of
regret for unknown $\mu_0$.  The proof is very similar to the one for
Theorem~\ref{thm:upper} and is also left for the appendix.


\section{The Adversarial Setting}
\label{sec:adversarial}

Unlike the stochastic case, in the adversarial multi-armed bandit
setting we do not make any assumptions about how the rewards are
generated.  Instead, we analyze a learner's worst-case performance
over all possible sequences of rewards $(X_{t,i})$.  In effect, we are
treating the environment as an adversary that has intimate knowledge
of the learner's strategy and will devise a sequence of rewards that
maximizes regret.  To preserve some hope of succeeding, however, the
learner is allowed to behave randomly: in each round it can randomize
its choice of arm $I_t$ using a distribution it constructs; the
adversary cannot influence nor predict the result of this random
choice.

Our goal is, as before, to satisfy the
constraint~\eqref{eq:constraint} while bounding the
regret~\eqref{eq:regret} with high probability (the randomness comes
from the learner's actions).  We assume that the default arm has a
fixed reward: $X_{t,0}=\mu_0\in[0,1]$ for all $t$; the other arms'
rewards are generated adversarially in $[0,1]$.  The constraint to be
satisfied then becomes $\sum_{s=1}^t X_{s,I_s} \ge (1-\alpha)\mu_0t$
for all $t$.

\paragraph*{Safe-playing strategy:}
We take any standard any-time high probability algorithm for
adversarial bandits and adapt it to play as usual when it is safe to
do so, i.e.\ when
$Z_t \ge \sum_{s=1}^{t-1}X_{s,I_s} - (1-\alpha)\mu_0t \ge 0$.
Otherwise it should play $I_t=0$.  To demonstrate a regret bound, we
only require that the bandit algorithm satisfy the following
requirement.
\begin{definition}
  An algorithm $\mathcal{A}$ is $\hat{R}_t^\delta$-\emph{admissible}
  ($\hat{R}_t^\delta$ sub-linear) if for any $\delta$, in the
  adversarial setting it satisfies
  \begin{align*}
    \P{\forall t\in\set{1,2,\dotsc}, R_t\le \hat{R}_t^\delta} \ge 1-\delta.
  \end{align*}
\end{definition}

Note that this performance requirement is stronger than the typical
high probability bound but is nevertheless achievable.  For example,
\citet{Neu2015} states the following for the any-time version of their
algorithm: given any time horizon $n$ and confidence level $\delta$,
$\P{R_n\le \hat{R}'_n(\delta)}\ge 1- \delta$ for some sub-linear
$\hat{R}'_t(\delta)$.  If we let
$\hat{R}_{t}^\delta = \hat{R}'_t(\delta/2t^2)$ then
$ \P{R_t\le \hat{R}_t^\delta}\ge 1- \frac{\delta}{2t^2} $ holds for
any fixed $t$.  Since the algorithm does not require $n$ and $\delta$
as input, a union bound shows it to be $\hat{R}_t^\delta$-admissible.

Having satisfied ourselves that there are indeed algorithms that meet
our requirements, we can prove a regret guarantee for our safe-playing
strategy.

\begin{restatable}{theorem}{thmadvupper}\label{thm:adv-upper}
  Any $\hat{R}_t^\delta$-admissible algorithm $\mathcal{A}$, when
  adapted with our safe-playing strategy, satisfies the
  constraint~\eqref{eq:constraint} and has a regret bound of
  $ R_n\le t_0 + \hat{R}_{n}^\delta$ with probability at least
  $1-\delta$ where
  $t_0 = \max\set{t\given\alpha\mu_0 t\le \hat{R}_{t}^\delta + \mu_0}$.
\end{restatable}

\begin{corollary}\label{cor:adv-upper-2}
  The any-time high probability algorithm of \citet{Neu2015} adapted
  with our safe-playing strategy gives
  $\hat{R}_t^\delta = 7\sqrt{Kt\log K}\log(4t^2/\delta)$ and
  \begin{align*}
    R_n\le 7\sqrt{Kn\log K}\log(4n^2/\delta) + \frac{49K\log K}{\alpha^2\mu_0^2}\log^2\frac{4n^2}{\delta} 
  \end{align*}
  with probability at least $1-\delta$.
\end{corollary}

Corollary~\ref{cor:adv-upper-2} shows that a strategy similar to that
of \cref{alg:cucb} also works for the adversarial setting.  However,
we pay a higher regret penalty to satisfy the constraint:
$O\paren*{\frac{KL^2}{{(\alpha\mu_0)}^2}}$ rather than the
$O\paren*{\frac{KL}{\alpha\mu_0}}$ we had in the stochastic setting.
Whether this is because
\begin{enumerate*}[label={(\roman*)}]
\item our algorithm is sub-optimal,
\item the analysis is not tight, or
\item there is some intrinsic hardness in the non-stochastic setting
\end{enumerate*}
is still not clear and remains an interesting open problem.


\section{Lower Bound on the Regret}
\label{sec:lower-bound}


We now present a worst-case lower bound where $\alpha$, $\mu_0$ and $n$ are fixed, but the mean rewards are free to change.
For any vector $\mu\in[0,1]^K$, we will write
$\mathds{E}_\mu$  to denote expectations  under the environment where all arms have
normally-distributed unit-variance rewards and means $\mu_i$ (i.e.,
the fixed value $\mu_0$ is the mean reward of arm 0 and the components
of $\mu$ are the mean rewards of the other arms).
We assume normally distributed noise for simplicity: Other subgaussian distributions
work identically as long as the subgaussian parameter can be kept fixed independently of the mean rewards.

\begin{restatable}{theorem}{thmminimaxlb}\label{thm:minimax-lb}
Suppose for any $\mu_i\in[0,1]$ ($i>0$) and $\mu_0$ satisfying
\begin{align*}
  \min\set{\mu_0,1-\mu_0}
  &\geq \max\set*{1/2\sqrt{\alpha}, \sqrt{e+1/2}}\sqrt{K/n},
\end{align*}
an algorithm satisfies
$\Ex_\mu\sum_{t=1}^n X_{t,I_t} \geq (1-\alpha)\mu_0 n $.  Then there
is some $\mu\in[0,1]^K$ such that its expected regret satisfies
$\Ex_\mu R_n \geq B$ where
\begin{equation}
  \label{eq:def-lower-bound}
  B=
\max \set*{
  \frac{K}{(16e+8)\alpha \mu_0},
  \frac{\sqrt{Kn}}{\sqrt{16e+8}}
  }.
\end{equation}
\end{restatable}

Theorem~\ref{thm:minimax-lb} shows that our algorithm for the
stochastic setting is near-optimal (up to a logarithmic factor $L$) in
the worst case.  A problem-dependent lower bound for the stochastic
setting would be interesting but is left for future work.  Also note
that in the lower bound we only use
$\Ex_\mu\sum_{t=1}^n X_t \geq (1-\alpha)n\mu_0$ for the last round
$n$, which means that the regret guarantee cannot be improved if we
only care about the last-round budget instead of the anytime budget.
In practice, however, enforcing the constraint in all rounds will
generally lead to significantly worse results because the algorithm
cannot explore early on.  This is demonstrated empirically in
\cref{sec:experiments}, where we find that the Unbalanced MOSS
algorithm performs very well in terms of the expected regret, but does
not satisfy the constraint in early rounds.

\begin{remark}
The theorem above almost
follows from the lower bound given by \citet{Lat15}, but in that paper $\mu_0$ is 
unknown, while here it may be known. This makes our result strictly stronger, as the lower bound is the same up to 
constant factors.
\end{remark}


\section{Experiments}
\label{sec:experiments}
\pgfplotsset{cycle list={{red,dotted}, {green!50!black}, {blue,dashed}, {black},{blue}}}
\pgfplotsset{every axis plot/.append style={line width=2pt}}

\pgfplotstableread[comment chars={\%}]{data/exp1.txt}{\tableFirst}
\pgfplotstableread[comment chars={\%}]{data/exp2.txt}{\tableSecond}

\newcommand{\defaultaxis}{
  legend cell align=left,
  scaled ticks=false,
  tick label style={/pgf/number format/fixed},
  compat=newest
}

We evaluate the performance of Conservative UCB compared to UCB and Unbalanced MOSS~\cite{Lat15} using simulated data in two regimes. 
In the first we fix
the horizon and sweep over $\alpha \in [0,1]$ to show the degradation of the average regret of Conservative UCB relative
to UCB as the constraint becomes harsher ($\alpha$ close to zero). In the second regime we fix $\alpha = 0.1$ and
plot the long-term average regret, showing that Conservative UCB is eventually nearly as good as UCB, despite the constraint. 
Each data point is an average of $N \approx 4000$ i.i.d.\ samples, which makes error bars too small to see.
All code and data will be made available in any final version.
Results are shown for both versions of Conservative UCB\@: The first knows the mean $\mu_0$ of the default arm while
the second does not and must act more conservatively while learning this value. As predicted by the theory, the difference
in performance between these two versions of the algorithm is relatively small, but note that even when $\alpha = 1$ the algorithm
that knows $\mu_0$ is performing better because this knowledge is useful in the unconstrained setting. This is also true of the BudgetFirst
algorithm, which is unconstrained when $\alpha = 1$ and exploits its knowledge of $\mu_0$ to eliminate the default arm. This algorithm is so conservative
that even when $\alpha$ is nearly zero it must first build a significant budget.
We tuned the Unbalanced MOSS algorithm with the following parameters.
\eq{
B_0 &= \frac{nK}{\sqrt{nK} + \frac{K}{\alpha \mu_0}} &
B_i &= B_K = \sqrt{nK} + \frac{K}{\alpha \mu_0}\,.
}
The quantity $B_i$ determines the regret of the algorithm with respect to arm $i$ up to constant factors, and must be chosen to lie 
inside the Pareto frontier given by \citet{Lat15}.
It should be emphasised that Unbalanced MOSS does \emph{not} constraint the return except for
the last round, and has no high-probability guarantees. 
This freedom allows it to explore early, which gives it a significant advantage over the highly constrained Conservative UCB\@.
Furthermore, it also requires $B_0,\ldots,B_K$ as inputs, which means that $\mu_0$ must be known in advance. 
The mean rewards in both experiments are $\mu_0 = 0.5$, $\mu_1 = 0.6$, $\mu_2 = \mu_3 = \mu_4 = 0.4$, which means that the default
arm is slightly sub-optimal. 

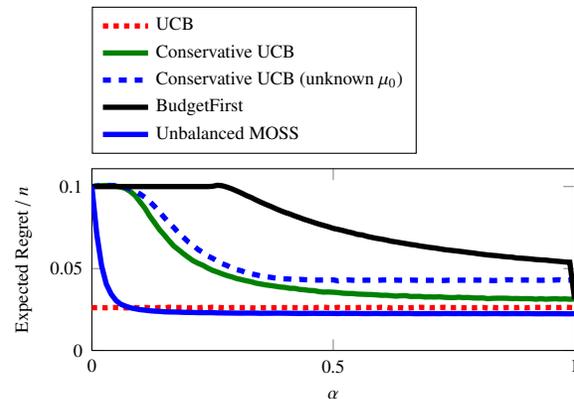
\begin{figure}[H]
  \begin{center}
    \begin{tikzpicture}[font=\scriptsize]
      \begin{axis}[\defaultaxis, xmin=0, ymin=0, xmax=1, width=8cm,
        height=4cm, xtick={0,0.5,1}, xlabel={$\alpha$}, legend
        style={anchor=south west,at={(axis cs:0,0.12)}},
        ylabel={Expected Regret / $n$}]
        
        \addplot+[] table[x expr=\thisrowno{0},y
        expr=\thisrowno{1}/10000] \tableSecond; \addlegendentry{UCB};
        \addplot+[] table[x expr=\thisrowno{0},y
        expr=\thisrowno{2}/10000] \tableSecond;
        \addlegendentry{Conservative UCB} \addplot+[] table[x
        expr=\thisrowno{0},y expr=\thisrowno{3}/10000] \tableSecond;
        \addlegendentry{Conservative UCB (unknown $\mu_0$)}
        \addplot+[] table[x expr=\thisrowno{0},y
        expr=\thisrowno{4}/10000] \tableSecond;
        \addlegendentry{BudgetFirst} \addplot+[] table[x
        expr=\thisrowno{0},y expr=\thisrowno{5}/10000] \tableSecond;
        \addlegendentry{Unbalanced MOSS}
      \end{axis}
    \end{tikzpicture}
  \end{center}
  \vspace{-0.75cm}
  \caption{\label{fig:alpha}Average regret for varying $\alpha$ and
    $n = 10^4$ and $\delta = 1/n$}
\end{figure}

\vspace{-0.5cm}

\begin{figure}[H]
  \begin{center}
    \begin{tikzpicture}[font=\scriptsize]
      \begin{axis}[\defaultaxis, xmin=100, ymin=0, xmax=100000,
        width=8cm, height=4cm, scaled x ticks=false,
        xtick={2000,50000,100000}, xlabel={$n$}, ylabel={Expected Regret
          / $n$}]
        
        \addplot+[] table[x expr=\thisrowno{0},y
        expr=\thisrowno{1}/\thisrowno{0}] \tableFirst; \addplot+[]
        table[x expr=\thisrowno{0},y expr=\thisrowno{2}/\thisrowno{0}]
        \tableFirst; \addplot+[] table[x expr=\thisrowno{0},y
        expr=\thisrowno{3}/\thisrowno{0}] \tableFirst; \addplot+[]
        table[x expr=\thisrowno{0},y expr=\thisrowno{4}/\thisrowno{0}]
        \tableFirst; \addplot+[] table[x expr=\thisrowno{0},y
        expr=\thisrowno{5}/\thisrowno{0}] \tableFirst;
      \end{axis}
    \end{tikzpicture}
  \end{center}
  \vspace{-0.75cm}
  \caption{Average regret as $n$ varies with $\alpha = 0.1$ and $\delta = 1/n$}
\end{figure}
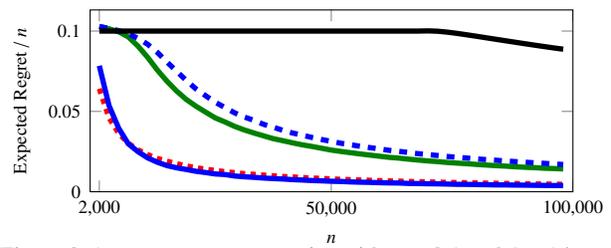

\vspace{-0.5cm}


\section{Conclusion}
We introduced a new family of multi-armed bandit frameworks motivated
by the requirement of exploring conservatively to maintain revenue.
We also demonstrated various strategies that act effectively while
maintaining such constraints.  We expect that similar strategies
generalize to other settings, like contextual bandits and
reinforcement learning.  We want to emphasize that this is just the
beginning of a line of research that has many potential applications.
We hope that others will join us in improving the current results,
closing open problems, and generalizing the model so it is more
widely applicable.

\renewcommand{\bibsection}{\subsection*{\refname}}

\bibliographystyle{abbrvnat}
\bibliography{references} 

\appendix

\twocolumn[
\thispagestyle{plain}
\centering\textbf{\Large Appendix}
\vspace{2\baselineskip}
]
\section{Proof of Theorem~\ref{thm:upper}}
\label{sec:proof-thm-upper}

\thmupper*{}

\begin{proof}
  By \cref{remark:choosing-psi}, with probability $\P{F}\geq 1-\delta$
  the confidence intervals are valid for all $t$ and all arms
  $i\in\set{1,\dotsc,K}$:
  \begin{align*}
    \abs{\hat{\mu}_i(t-1)-\mu_i}& \le \sqrt{\psi^\delta(T_i(t-1))/T_i(t-1)} \\& \le \sqrt{L/T_i(t-1)};
  \end{align*}
  we will henceforth assume that this is the case (i.e.\ that $F$
  holds).  By the definition of the confidence intervals and by the
  construction of \cref{alg:cucb} we immediately satisfy the
  constraint
  \begin{equation*}
    \sum_{t=1}^n \mu_{I_t} \geq (1 - \alpha) n \mu_0 \qquad\text{for all } n.
  \end{equation*}
  We now bound the regret.  Let $i > 0$ be the index of a sub-optimal
  arm and suppose $I_t = i$. Since the confidence intervals are valid,
  \begin{align*}
    \mu^* \leq \theta_i(t)
    &\leq \hat \mu_i(t-1) + \sqrt{L/T_i(t-1)} \\
    &\leq \mu_i + 2\sqrt{L/T_i(t-1)}\,,
  \end{align*}
  which implies that arm $i$ has not been chosen too often; in
  particular we obtain
  \begin{equation}
    \label{eq:ub-times-chosen}
    T_i(n) \leq T_i(n-1) + 1 \leq \frac{4L}{\Delta_i^2} + 1.
  \end{equation}
  and the regret satisfies
  \begin{equation*}
    \widetilde{R}_n
    = \sum_{i=0}^K T_i(n)\Delta_i
    \le \sum_{i>0:\Delta_i>0} \paren*{\frac{4L}{\Delta_i}+\Delta_i}
      + T_0(n)\Delta_0.
  \end{equation*}
  If $\Delta_0 = 0$ then the theorem holds trivially; we therefore
  assume that $\Delta_0 > 0$ and find an upper bound for $T_0(n)$.

  Let $\tau=\max\set{t\le n\given I_t=0}$ be the last round in which
  the default arm is played.  Since $F$ holds and
  $\theta_0(t) = \mu_0 < \mu^* < \max_i\theta_i(t)$, it follows that
  $J_t = 0$ is never the UCB choice; the default arm was only played
  because $\xi_\tau < 0$:
  \begin{equation}
    \label{eq:known-mu0-budget}
    \sum_{i=0}^K T_i(\tau-1)\lambda_i(\tau)+\lambda_{J_\tau}(\tau)  - (1-\alpha)\mu_0\tau < 0
  \end{equation}
  By dropping $\lambda_{J_\tau}(\tau)$, replacing $\tau$ with
  $\sum_{i=0}^K T_i(\tau-1)+1$, and rearranging the terms
  in~\eqref{eq:known-mu0-budget}, we get
  \begin{align}
    \MoveEqLeft[1] \alpha T_0(\tau-1)\mu_0 \nonumber\\
    &< (1-\alpha)\mu_0  + \sum_{i=1}^K T_i(\tau-1)\left((1-\alpha)\mu_0-\lambda_i(\tau) \right) \nonumber\\
      \begin{split}
        &\le (1-\alpha)\mu_0 
        \\&\quad + \sum_{i=1}^K T_i(\tau-1)\paren*{(1-\alpha)\mu_0-\mu_i + \sqrt{\frac{L}{T_i(\tau-1)}}}
      \end{split} \nonumber\\
    &\le 1 + \sum_{i=1}^K S_i  \,. \label{eq:known-mu0-budget-3}  
  \end{align}
  where $a_i = (1-\alpha)\mu_0-\mu_i$ and
  \begin{align*}
    S_i &= T_i(\tau-1)\cdot\paren[\Big]{(1-\alpha)\mu_0
          - \mu_i + \sqrt{L/T_i(\tau-1)}} \\
        &= a_i T_i(\tau-1) + \sqrt{LT_i(\tau-1)}
  \end{align*}
  is a bound on the decrease in $\xi_t$ in the first $\tau-1$ rounds
  due to choosing arm $i$.  We will now bound $S_i$ for each $i>0$.

  The first case is $a_i\ge 0$, i.e.\
  $\Delta_i\ge \Delta_0 +\alpha\mu_0$.
  Then~\eqref{eq:ub-times-chosen} gives
  $T_i(\tau-1)\le 4L/\Delta_i^2 + 1$ and we get
  \begin{equation}
    \label{eq:known-mu0-case1}
    S_i \le \frac{4La_i}{\Delta_i^2} + \frac{2L}{\Delta_i} + 2
    \le \frac{6L}{\Delta_i} + 2 \,.  
  \end{equation}
  The other case is $a_i < 0$, i.e.\
  $\Delta_i< \Delta_0 +\alpha\mu_0$.  Then
  \begin{equation}
    \label{eq:known-mu0-case2}
    S_i \le \sqrt{LT_i(\tau-1)} \le \frac{2L}{\Delta_i} + 1,
  \end{equation}
  and by using $ax^2+bx \le -b^2/4a$ for $a<0$ we have
  \begin{equation}
    \label{eq:known-mu0-case3}
    S_i \le -\frac{L}{4a_i} = \frac{L}{4(\Delta_0 +\alpha\mu_0 - \Delta_i)}.
  \end{equation} 
  Summarizing~\labelcref{eq:known-mu0-case1,eq:known-mu0-case2,eq:known-mu0-case3}
  gives
  \begin{equation*}
    S_i  \le \frac{6L}{\max\{\Delta_i,\Delta_0-\Delta_i\}} + 2\,.
  \end{equation*}
  Continuing from~\eqref{eq:known-mu0-budget-3}, we get
  \begin{align*}
    T_0(n)& =T_0(\tau-1)+1\\
          & \le \frac{2K+2}{\alpha\mu_0} + \frac{1}{\alpha\mu_0}
            \sum_{i=1}^K \frac{6L}{\max\{\Delta_i,\Delta_0-\Delta_i\}} \,.
  \end{align*}
  We can now upper bound the regret by
  \begin{multline}
    \widetilde{R}_n
    \le \sum_{i>0:\Delta_i>0} \left(\frac{4L}{\Delta_i}+\Delta_i\right)
    + \frac{2(K+1)\Delta_0}{\alpha\mu_0} \\
    + \frac{6L}{\alpha\mu_0}
    \sum_{i=1}^K \frac{\Delta_0}{\max\{\Delta_i,\Delta_0-\Delta_i\}} \,.
    \tag{\ref{eq:pd-upper}}
  \end{multline}
  
  We will now show~\eqref{eq:pind-upper}.  To bound the regret due to
  the non-default arms, Jensen's inequality gives
  \begin{equation*}
    \paren*{\sum_{i>0}T_i(n)\Delta_i}^2
    \leq m^2 \sum_{i>0} \frac{T_i(n)}{m} \Delta_i^2 ,
  \end{equation*}
  where $m\leq n$ is the number of times non-default arms were chosen.
  Combining this with $\Delta_i^2 \leq 4L/T_i(n)$ for sub-optimal arms
  from~\eqref{eq:ub-times-chosen} gives
  \begin{equation*}
    \sum_{i>0} T_i(n)\Delta_i \leq 2\sqrt{mKL} \in O(\sqrt{nKL}).
  \end{equation*}
  To bound the regret due to the default arm, observe that
  $\max\{\Delta_i,\Delta_0-\Delta_i\}\ge \Delta_0/2$ and thus
  $T_0(n)\Delta_0 \in O(KL/\alpha\mu_0)$.  Combining these two bounds
  gives~\eqref{eq:pind-upper}.
\end{proof}

\section{Proof of \cref{thm:unknown-mu0}}
\label{sec:proof-thm-unknown-mu0}

\thmunknownmuzero*

\begin{proof}
  We proceed very similarly to the proof of \cref{thm:upper} in
  \cref{sec:proof-thm-upper}.  As we did there, we assume that $F$
  holds: the confidence intervals are valid for all rounds and all
  arms (including the default), which happens with probability
  $\P{F}\geq 1-\delta$.

  To show that the modified algorithm satisfies the
  constraint~\eqref{eq:stoch-constraint}, we write the
  budget~\eqref{eq:stoch-budget} as
  \begin{align*}
    \widetilde{Z}_t 
    &= \sum_{i=1}^K T_i(t-1)\mu_i + \mu_{J_t} + (T_0(t-1) - (1-\alpha)t)\mu_0
  \end{align*}
  when the UCB arm $J_t$ is chosen and show that it is indeed
  lower-bounded by
  \begin{multline*}
    \xi'_t = \sum_{i=1}^K T_i(t-1)\lambda_i(t)+\lambda_{J_t}(t) \\ 
    + (T_0(t-1) - (1-\alpha)t)\theta_0(t) \,.
    \tag{\ref{eq:unknown-mu0-budget}}
  \end{multline*}
  This is apparent if $T_0(t-1) < (1-\alpha)t$, since the last term
  in~\eqref{eq:unknown-mu0-budget} is then negative and
  $\theta_0(t)\geq\mu_0$.  On the other hand, if $T_0(t-1) \geq
  (1-\alpha)t$ then the constraint is still satisfied:
  \begin{equation*}
    \sum_{s=1}^t \mu_{I_s} \geq T_0(t-1)\mu_0 \geq (1-\alpha)\mu_0t.
  \end{equation*}

  We now upper-bound the regret.  As in the earlier proof, we can show
  that for any arm $i>0$ with $\Delta_i>0$ we have
  $T_i(n)\le 4L/\Delta_i^2 + 1$.  If this also holds for $i=0$ or if
  $\Delta_0=0$ then
  $\widetilde{R}_n\le \sum_{i:\Delta>0}(4L/\Delta_i+\Delta_i)$ and the
  theorem holds trivially.  From now on we only consider the case when
  $\Delta_0 > 0$ and $T_0(n) > 4L/\Delta_0^2 + 1$.  As before, we will
  proceed to upper-bound $T_0(n)$.

  Let $\tau$ be the last round in which $I_\tau = 0$.  We can ignore
  the possibility that $J_\tau = 0$, since then the above bound on
  $T_i(n)$ would apply even to the default arm, contradicting our
  assumption above.  Thus we can assume that the default arm was
  played because $\xi'_\tau<0$:
  \begin{multline*}
    \sum_{i=1}^K T_i(\tau-1)\lambda_i(\tau)+\lambda_{J_\tau}(\tau) \\ 
    + \paren[\big]{T_0(\tau-1) - (1-\alpha)\tau}\,\theta_0(\tau) < 0 \, ,
  \end{multline*}
  in which we drop $\lambda_{J_\tau}(\tau)$, replace $\tau$ with
  $\sum_{i=0}^K T_i(\tau-1)+1$, and rearrange the terms to get
  \begin{multline}
    \label{eq:unknown-mu0-budget-2}
    \alpha T_0(\tau-1)\theta_0(\tau)<(1-\alpha)\theta_0(\tau) \\
    + \sum_{i=1}^K T_i(\tau-1)\paren[\big]{(1-\alpha)\theta_0(\tau)-\lambda_i(\tau)} \,.
  \end{multline}
  We lower-bound the left-hand side of~\eqref{eq:unknown-mu0-budget-2}
  using $\theta_0(\tau) \ge \mu_0$, whereas we upper-bound the
  right-hand side using
  \begin{align*}
    \theta_0(\tau) \le \mu_0+\sqrt{\frac{L}{T_0(\tau-1)}}
    \le \mu_0+\frac{\Delta_0}{2}\,,
  \end{align*}
  which comes from $T_0(\tau-1)\ge 4L/\Delta_0^2$.  Combining these
  in~\eqref{eq:unknown-mu0-budget-2} with the lower confidence bound
  $\lambda_i(\tau) \ge \mu_i-\sqrt{L/T_i(\tau-1)}$ gives
  \begin{align}
    \alpha\mu_0T_0(\tau-1)
    &< (1-\alpha)\paren*{\mu_0+\frac{\Delta_0}{2}} \nonumber\\
    &\hspace{1.5em} + \sum_{i=1}^K T_i(\tau-1)\Bigg((1-\alpha)
      \paren*{\mu_0+\frac{\Delta_0}{2}} \nonumber\\
    &\hspace{8.5em} - \mu_i +\sqrt{\frac{L}{T_i(\tau-1)}} \Bigg) \nonumber\\
    &= (1-\alpha)\paren*{\mu_0 + \frac{\Delta_0}{2}}  + \sum_{i=1}^K S_i \nonumber\\
    &\le 1 + \sum_{i=1}^K S_i  \,, \label{eq:unknown-mu0-budget-3}
  \end{align}
  where $a_i = (1-\alpha)(\mu_0+\Delta_0/2)-\mu_i$ and
  \begin{align*}
    S_i=a_i T_i(\tau-1) + \sqrt{LT_i(\tau-1)}
  \end{align*}
  is a bound on the decrease in $\xi'_t$ in the first $\tau-1$ rounds
  due to choosing arm $i$.  We will now bound $S_i$ for each $i>0$.

  Analogously to the previous proof, we get the bounds
  \begin{align}
    S_i &\leq \frac{6L}{\Delta_i} + 2, \quad\text{when } a_i \geq 0 \,;
          \label{eq:unknown-mu0-case1}\\
    S_i &\leq \frac{2L}{\Delta_i} + 1\,, \quad\text{otherwise; }
          \label{eq:unknown-mu0-case2}
          \intertext{and in the latter case, using $ax^2+bx \le -b^2/4a$ gives}
    S_i &\le -\frac{L}{4a_i}
          = \frac{L}{4\paren[\big]{(1+\alpha)\Delta_0/2
          + \alpha\mu_0 - \Delta_i}}\,. \label{eq:unknown-mu0-case3}
  \end{align}
  Summarizing~\labelcref{eq:unknown-mu0-case1,eq:unknown-mu0-case2,%
    eq:unknown-mu0-case3} gives
  \begin{align*}
    S_i & \le \frac{6L}{\max\set*{\Delta_i,
          24 \paren[\big]{(1+\alpha)\Delta_0/2 + \alpha\mu_0 - \Delta_i}}} + 2\\
        & \le \frac{7L}{\max\set{\Delta_i,\Delta_0-\Delta_i}} + 2\,.
  \end{align*}
  Continuing with~\eqref{eq:unknown-mu0-budget-3}, if
  $T_0(n) > \frac{4L}{\Delta_0^2}+1$, we get
  \begin{align*}
    T_0(n) &=T_0(\tau-1) + 1\\
           &\le \frac{2K+2}{\alpha\mu_0} + \frac{1}{\alpha\mu_0}
             \sum_{i=1}^K \frac{7L}{\max\{\Delta_i,\Delta_0-\Delta_i\}} \,.
  \end{align*}
  We can now upper bound the regret by
  \begin{multline}
    \widetilde{R}_n
    \le \sum_{i:\Delta_i>0} \left(\frac{4L}{\Delta_i}+\Delta_i\right)
    + \frac{2(K+1)\Delta_0}{\alpha\mu_0} \\
    + \frac{7L}{\alpha\mu_0} \sum_{i=1}^K
    \frac{\Delta_0}{\max\{\Delta_i,\Delta_0-\Delta_i\}} \,.
    \tag*{(\ref{eq:unknown-mu0-pd-upper}) \qed}
  \end{multline}
  \renewcommand{\qedsymbol}{}
\end{proof}

\section{Proof of Theorem~\ref{thm:adv-upper}}

\thmadvupper*

\begin{proof}[Proof of Theorem~\ref{thm:adv-upper}]
  It is clear from the description of the safe-playing strategy that
  it is indeed safe: the constraint~\eqref{eq:constraint} is always
  satisfied.

  The algorithm plays safe when the following quantity, which is a
  lower bound on the budget $Z_t$, is negative:
  \begin{equation*}
    Z'_t = Z_t - X_{t,I_t} = \sum_{s=1}^{t-1} X_{s,I_s} -
    (1-\alpha)\mu_0t
  \end{equation*}
  To upper bound the regret, consider only the rounds in which our
  safe-playing strategy does not interfere with playing
  $\mathcal{A}$'s choice of arm.  Then with probability $1-\delta$,
  \begin{align*}
    \max_{i\in\set{0,\dotsc,K}} \sum_{s=1}^t \ind{Z'_s\ge 0}(X_{s,i} - X_{s,I_s})
    \le \hat{R}^\delta_{B(t)}
  \end{align*}
  where $B(t)=\sum_{s=1}^t \ind{Z'_s\ge 0}$.  Let $\tau$ be the last
  round in which the algorithm plays safe.
  \begin{align*}
    \MoveEqLeft[1.5] \mu_0 B(\tau-1) \\
    &\le \max_i\sum_{s=1}^{\tau-1} \ind{Z'_s\ge 0} X_{s,i}\\
    &\le \hat{R}_{B(\tau-1)}^\delta + \sum_{s=1}^{\tau-1} \ind{Z'_s\ge 0} X_{s,I_s}\\
    &= \hat{R}_{B(\tau-1)}^\delta
      + \sum_{s=1}^{\tau-1} X_{s,I_s} - \mu_0(\tau-1-B(\tau-1))\\
    &\le \hat{R}_{B(\tau-1)}^\delta + (1-\alpha)\mu_0\tau - \mu_0(\tau-1-B(\tau-1)) \,,
  \end{align*}
  which indicates $\alpha\mu_0\tau \le \hat{R}_\tau^\delta + \mu_0$ and thus
  $\tau \le t_0$.  It follows that $R_n \le t_0 + \hat{R}_{n}^\delta$.
\end{proof}

\section{Proof of Theorem~\ref{thm:minimax-lb}}

\thmminimaxlb*

\begin{proof}[Proof of Theorem~\ref{thm:minimax-lb}]
  Pick any algorithm. We want to show that the algorithm's regret on
  some environment is at least as large as $B$. If $\Ex_{\mu}R_n > B$
  for some $\mu\in [0,1]^K$, there is nothing to be proven. Hence,
  without loss of generality, we can assume that the algorithm is
  \emph{consistent} in the sense that $\Ex_{\mu}R_n \leq B$ for all
  $\mu\in[0,1]^K$.

  For some $\Delta>0$, define environment $\mu \in \Real^K$ such that
  $\mu_i=\mu_0-\Delta$ for all $i\in [K]$.  For now, assume that
  $\mu_0$ and $\Delta$ are such that $\mu_i\ge 0$; we will get back to
  this condition later.  Also define environment $\mu^{(i)}$ for each
  $i=1,\dotsc,K$ by
  \begin{align*}
    \mu^{(i)}_j =
    \begin{cases}
      \mu_0 + \Delta, &\text{for } j = i\,;\\
      \mu_0 - \Delta, &\text{otherwise.}
    \end{cases}
  \end{align*}
  In this proof, we use $T_i = T_i(n)$ to denote the number of times
  arm $i$ was chosen in the first $n$ rounds.  We distinguish two
  cases, based on how large the exploration budget is.

  \textbf{Case 1: $\displaystyle
    \alpha \ge \frac{\sqrt{K}}{\mu_0 \sqrt{(16e+8)n}}$.}

  In this case, $B = \frac{\sqrt{Kn}}{\sqrt{16e+8}}$ and
  we use $\Delta = (4e+2)B/n$.  For each $i\in[K]$ define
  event $A_i = \set{T_i \le 2B/\Delta}$.  First we prove that
  $\Pr_\mu(A_i) \ge 1/2$:
  \begin{align*}
    \Pr_\mu\set{T_i\le 2B/\Delta}
    &= 1-\Pr_\mu\set{T_i > 2B/\Delta} \\
    &\ge 1-\frac{\Delta\Ex_\mu[T_i]}{2B} \\
    &\ge 1-\frac{\Ex_\mu[R_n]}{2B}
    \ge \frac{1}{2}\,.
  \end{align*}
  Next we prove that $\Pr_{\mu^{(i)}}(A_i) \le 1/4e$:
  \begin{align*}
    \Pr_{\mu^{(i)}}\set{T_i\le 2B/\Delta}
    &= \Pr_{\mu^{(i)}}\set{n-T_i\ge n-2B/\Delta} \\
    &\le \frac{\Ex_{\mu^{(i)}}[n-T_i]}{n-2B/\Delta} \\
    &\le \frac{B}{\Delta n-2B}
    = \frac{1}{4e}\,.
  \end{align*}
  Note that $\mu$ and $\mu^{(i)}$ differ only in the $i$th component:
  $\mu_i = \mu_0-\Delta$ whereas $\mu^{(i)}_i = \mu_0+\Delta$.  Then
  the KL divergence between the reward distributions of the $i$th arms
  is $\mathrm{KL}(\mu_i,\mu^{(i)}_i) = (2\Delta)^2/2 = 2\Delta^2$.
  Define the \emph{binary relative entropy} to be
  \begin{align*}
    d(x,y) &= x\log\frac{x}{y} + (1-x)\log\frac{1-x}{1-y};
  \end{align*}
  it satisfies $d(x,y)\geq(1/2)\log(1/4y)$ for $x\in[1/2, 1]$ and
  $y\in(0,1)$.  By a standard change of measure argument \citep[see,
  e.g.,][Lemma~1]{Kaufman2015} we get that
  \begin{align*}
  \Ex_\mu[T_i]\cdot\mathrm{KL}(\mu_i;\mu^{(i)}_i)
    &\ge d(\Pr_\mu(A_i), \Pr_{\mu^{(i)}}(A_i)) \\
    &\geq \frac{1}{2}\log\frac{1}{4(1/4e)} = \frac{1}{2}
  \end{align*}
  and so $\Ex_\mu[T_i] \ge 1/4\Delta^2$ for each $i\in [K]$. Hence
  \begin{align*}
    \Ex_\mu[R_n] = \Delta\sum_{i\in[K]}\Ex_\mu[T_i] \ge \frac{K}{4\Delta} = \frac{\sqrt{Kn}}{\sqrt{16e+8}} = B\,.
  \end{align*}

  \textbf{Case 2: $\displaystyle
    \alpha < \frac{\sqrt{K}}{\mu_0 \sqrt{(16e+8)n}}$.}

  In this case, $B = \frac{K}{(16e+8)\alpha \mu_0}$ and
   we use $\Delta=K/4\alpha\mu_0 n$.  For each $i$ define the
  event $A_i=\set{T_i\le 2\alpha\mu_0 n/\Delta}$.  First we prove that
  $\Pr_\mu(A_i) \ge 1/2$:
  \begin{align*}
    \Pr_\mu\set{T_i\le 2\alpha\mu_0 n/\Delta}
    &= 1 - \Pr_\mu\set{T_i > 2\alpha\mu_0 n/\Delta} \\
    &\ge 1-\frac{\Delta\Ex_\mu[T_i]}{2\alpha\mu_0 n} \\
    &\ge 1-\frac{\Ex_\mu[R_n]}{2\alpha\mu_0 n}
    \ge \frac{1}{2}\,,
  \end{align*}
  where we use the fact that
  \begin{align*}
    \Ex_\mu[R_n] &= n\mu_0-\Ex_\mu\bracket[\Big]{\sum_{t=1}^n X_{t,I_t}} \\
    &\le n\mu_0 - (1-\alpha)\mu_0 n = \alpha\mu_0 n.
  \end{align*}
  Next, we show that $\Pr_{\mu^{(i)}}(A_i) < 1/4e$:
  \begin{align*}
    \MoveEqLeft \Pr_{\mu^{(i)}}\set{T_i\le 2\alpha\mu_0 n/\Delta} \\
    &= \Pr_{\mu^{(i)}}\set{n-T_i\ge n-2\alpha\mu_0 n/\Delta} \\
    &\le \frac{\Ex_{\mu^{(i)}}[n-T_i]}{n - 2\alpha\mu_0 n/\Delta} \\
    &\le \frac{B}{\Delta n-2\alpha\mu_0 n} \\
    & = \frac{K}{(4e+2)K-(32e+16)\alpha^2 \mu_0^2 n}
    < \frac{1}{4e}\,.
  \end{align*}
  As in the other case, we have $\Ex_\mu[T_i]>1/4\Delta^2$ for each
  $i\in [K]$.  Therefore
  \begin{align*}
  \Ex_\mu[R_n] &= \Delta\sum_{i\in[K]}\Ex_\mu[T_i]
                  > \frac{K}{4\Delta} = \alpha\mu_0 n,
  \end{align*}
  which contradicts the fact that $\Ex_\mu[R_n] \le \alpha\mu_0 n$.
  So there does not exist an algorithm whose worst-case regret is
  smaller than $B$.

  To summarize, we proved that
  \begin{align*}
  \Ex_\mu R_n \ge
    \begin{cases}
      \dfrac{\sqrt{Kn}}{\sqrt{16e+8}}, \quad\text{when }
      \alpha \ge \dfrac{\sqrt{K}}{\mu_0 \sqrt{(16e+8)n}} \\
      \dfrac{K}{(16e+8)\alpha \mu_0}, \quad\text{otherwise,}
    \end{cases}
  \end{align*}
  finishing the proof.
\end{proof}


\end{document}